\documentclass[11pt]{article}

\usepackage[preprint]{acl}

\usepackage{times}
\usepackage{latexsym}
\usepackage{amssymb}
\usepackage{amsthm}
\usepackage{amsmath}
\usepackage{booktabs}
\usepackage{siunitx}

\usepackage[T1]{fontenc}

\usepackage[utf8]{inputenc}

\usepackage{microtype}

\usepackage{inconsolata}

\usepackage{graphicx}
\usepackage{xcolor}

\newtheorem{proposition}{Proposition}
\newtheorem{definition}{Definition}[section]

\title{Initialisation Determines the Basin: Efficient Codebook Optimisation for Extreme LLM Quantization\thanks{Preprint. Under review.}}

\author{Ian W. Kennedy \\
  Department of Computer Science \\
  University of Sheffield \\
  Sheffield, UK \\
  \texttt{iwkennedy@sheffield.ac.uk} \\\And
  Nafise Sadat Moosavi \\
  Department of Computer Science \\
  University of Sheffield \\
  Sheffield, UK \\
  \texttt{n.s.moosavi@sheffield.ac.uk} \\}

\begin{document}
\maketitle

\begin{abstract}
Additive quantization enables extreme LLM compression with $O(1)$ lookup-table dequantization, making it attractive for edge deployment. Yet at 2-bit precision, it often fails catastrophically, even with extensive search and fine-tuning. We show that the dominant bottleneck is codebook initialisation. Greedy sequential initialisation frequently places the model in poor optimisation regions that subsequent beam search and PV-tuning struggle to overcome. We analyse this behaviour through the representational ratio $\rho = N/K^M$, which characterises the relationship between weight groups and codebook capacity, and propose OA-EM, an output-aware EM initialisation method using Hessian-weighted Mahalanobis distance. Across compression rates, search budgets, and three architectures (Llama 3.2 3B, Llama 3.1 8B, Qwen 2.5 3B), OA-EM consistently produces better solutions after PV-tuning and dominates the quality–compute frontier. The severity of the bottleneck scales with $\rho$: moderate at 3 bpp but extreme at 2 bpp, where poor initialisation can degrade perplexity by orders of magnitude. More broadly, our results highlight the importance of optimisation geometry in compressed model spaces, where initialisation can dominate subsequent search and fine-tuning. Our code is available at \footnote{\url{https://github.com/kenno94-IK/aqlm-oaem}}.
\end{abstract}

\section{Introduction}
Large language model deployment on consumer hardware requires aggressive weight compression \citep{gholami2022survey}. While 4-bit quantization is near-lossless \citep{frantar2022gptq,lin2024awq}, the 2-bit regime remains challenging: each parameter is encoded with only 4 possible values, leaving almost no room for approximation error. This regime is particularly relevant for the 3B--8B parameter range, where 2-bit compression enables deployment on single consumer GPUs and mobile edge devices: hardware where memory, not compute, is the binding constraint.

Two paradigms compete in this regime. \emph{Structured} methods---such as lattice codebooks \citep{tseng2024quipsharp}, trellis codes \citep{tseng2024qtip}, and grouped lattice VQ \citep{zhang2025glvq}---achieve strong perplexity via mathematically constrained codebook geometry, but require active computation (e.g.\ Babai rounding, matrix--vector multiplication) during inference. In contrast, \emph{free-form} additive methods \citep{egiazarian2024} learn unconstrained codebooks, enabling $O(1)$ lookup-table (LUT) dequantization with zero multiply-accumulate (MAC) operations per weight group---a pure memory read that is critical for edge deployment on ARM CPUs, microcontrollers, and mobile SoCs where ALU cycles, not memory bandwidth, are the primary bottleneck \citep{egiazarian2024}. Our work focuses on this free-form family.

Additive quantization \citep{egiazarian2024} encodes each group of $g$ weights as a sum of $M$ codewords, drawn from learned codebooks of $K{=}256$ entries each. When performance degrades at extreme compression, the common response is to increase computational effort: wider beam search, more epochs per layer, or larger calibration sets. 
We show that, in this regime, such strategies target the wrong bottleneck. The dominant factor is initialisation: search performed within a poorly initialised region of the solution space yields limited improvement.
The key insight is a \emph{regime transition} governed by the ratio of weight groups to representational capacity:

\begin{definition}[Representational Ratio]
\label{def:rho}
For $M$ additive codebooks with $K$ entries each and $N$ weight groups per layer, the representational ratio is $\rho = N / K^M$.
\end{definition}

\noindent When $\rho < 1$ (overcomplete), there are more representable points than weight groups and initialisation errors can often be absorbed. When $\rho > 1$ (undercomplete), weight groups compete for limited codebook capacity, and initial placement becomes critical. For Llama~3.2~3B:

\begin{table}[h]
\centering
\small
\begin{tabular}{lcccl}
\toprule
Rate & $M$ & $K^M$ & $\rho$ & Regime \\
\midrule
3\,bpp & 3 & $16.8\text{M}$ & ${\approx}\,0.07$ & Overcomplete \\
2\,bpp & 2 & $65\text{K}$ & ${\approx}\,18$ & Undercomplete \\
\bottomrule
\end{tabular}
\end{table}

\noindent This $256\times$ capacity reduction is not a gradual degradation---it is a qualitative change in behaviour. At 3\,bpp, greedy initialisation costs 0.65 perplexity points. At 2\,bpp, it leads to severe degradation on Llama~3.2~3B (WikiText-2 perplexity 352.39 at beam\,4, 60.61 at beam\,8 vs.\ 7.28 FP16), and even quadrupling the beam width to 16 reduces this only to 46.01.

To address this, we propose \textbf{OA-EM} (Output-Aware Expectation-Maximisation), which refines each codebook's initialisation via iterative EM \citep{dempster1977em} using Hessian-weighted Mahalanobis distance derived from calibration activations, directly minimising output reconstruction error rather than weight-space distance.

Beyond the algorithm itself, our primary contribution is the empirical and theoretical observation that \emph{initialisation strongly influences the optimisation trajectory of compressed models}. Whenever learned codebooks are initialised through greedy sequential fitting, the representational ratio $\rho$ predicts when initialisation becomes critical. We demonstrate this through three lines of evidence:

\emph{Basin persistence.} OA-EM's advantage persists after end-to-end PV-tuning \citep{malinovskii2024} across beam widths, epoch budgets, compression rates, model scales, and architectures. On Llama~3.2~3B, PV-tuning compresses a 43-point perplexity gap to 0.23, yet OA-EM remains better in every configuration.

\emph{Asymmetric search scaling.} Increasing beam width from 8 to 16 improves OA-EM (11.53 $\rightarrow$ 11.49 post-PV) but worsens the greedy baseline (11.76 $\rightarrow$ 12.01 post-PV), suggesting that additional search is beneficial primarily when initialisation is already good.

\emph{Pareto dominance.} OA-EM dominates the quality–compute frontier, achieving lower perplexity and better downstream accuracy at matched compute budgets.

\paragraph{Contributions.}
\begin{enumerate}
    \item We introduce the representational ratio $\rho$ and show that it predicts when quantization becomes sensitive to initialisation (\S\ref{sec:theory}).

    \item We propose OA-EM---output-aware EM with Hessian-weighted Mahalanobis distance---as a drop-in replacement for greedy initialisation (\S\ref{sec:oaem}).
      \item Through systematic post-PV analysis across beam widths, epoch budgets, compression rates (2\,bpp and 3\,bpp), and three models (Llama~3.2~3B, Llama~3.1~8B, Qwen~2.5~3B), we show that initialisation can lead to persistent optimisation differences whose severity scales with $\rho$ (\S\ref{sec:postpv}).

\end{enumerate}

\section{Related Work}

\paragraph{Scalar PTQ.}
A large body of work studies scalar post-training quantization (PTQ) for LLM compression, where each weight is assigned an independent low-bit representation. GPTQ \citep{frantar2022gptq} quantizes layer-wise using approximate second-order information; AWQ \citep{lin2024awq} protects salient weights via activation-based scaling. SmoothQuant \citep{xiao2023smoothquant} migrates quantization difficulty from activations to weights via per-channel scaling, and OmniQuant \citep{shao2024omniquant} learns weight clipping and equivalence transformations end-to-end. 
Such scalar methods degrade rapidly below 3\,bpp because scalar codes cannot efficiently use the information budget at extreme compression. LLM.int8() \citep{dettmers2022llmint8} demonstrated mixed-precision decomposition for 8-bit quantization. SpQR \citep{dettmers2024spqr} and SqueezeLLM \citep{kim2024squeezellm} extend sensitivity-aware principles to 3--4 bits. QLoRA \citep{dettmers2023qlora} combines 4-bit quantization with low-rank adaptation for efficient fine-tuning, demonstrating the practical demand for aggressive compression on consumer hardware. These scalar approaches share our motivation that weight sensitivity should guide compression, but operate in the regime where the initialisation problem we study does not arise.

\paragraph{Structured vector quantization for LLMs.}
Beyond scalar quantization, several works explore vector quantization with structured codebooks for extreme compression. QuIP \citep{chee2023quip} enabled 2-bit LLMs via incoherence processing. QuIP\# \citep{tseng2024quipsharp} introduced the randomised Hadamard transform and E8 lattice codebooks, because the codebook structure is fixed (not learned), QuIP\# avoids the initialisation problem we study. SpinQuant \citep{liu2025spinquant} learns rotation matrices to remove outliers and improve quantization accuracy; like QuIP\#, it transforms the weight distribution rather than learning codebooks, so the initialisation bottleneck does not apply. QTIP \citep{tseng2024qtip} achieves strong 2-bit results via trellis coded quantization.

\paragraph{Grouped lattice VQ.}
A recent line of work further develops structured codebooks through lattice-based constructions.
GLVQ \citep{zhang2025glvq} assigns each weight group a customised lattice codebook defined by a learnable generation matrix, achieving state-of-the-art 2-bit perplexity by combining structured codebook geometry with per-group adaptability. GLVQ sidesteps the free-form initialisation trap entirely: Babai rounding provides a closed-form nearest-lattice-point solution, eliminating the combinatorial assignment problem. However, this advantage comes with an inference trade-off: lattice dequantization requires $O(d)$ floating-point MAC operations per weight group at runtime (matrix--vector multiplication via the generation matrix), whereas free-form additive codes (AQLM) use pre-computed lookup tables (LUTs) requiring exactly zero MACs, a pure $O(1)$ memory read per group. As demonstrated by \citet{egiazarian2024}, this distinction translates to significant real-world speedups on CPU and edge hardware where the lack of dedicated tensor cores makes runtime MAC operations prohibitively expensive. On ARM CPUs, microcontrollers, and low-power inference accelerators where ALU cycles are the primary bottleneck, LUT-based dequantization remains substantially faster. Our work is orthogonal to GLVQ: while GLVQ bypasses free-form codebooks to avoid optimisation traps, we show that free-form codebooks are highly capable at 2\,bpp when the initialisation basin is corrected via OA-EM, preserving the LUT inference pathway.

\paragraph{Additive quantization for LLMs.}
Additive quantization represents weights as sums of codewords from multiple codebooks, allowing a much larger set of representable values than scalar quantization at the same bitrate. AQLM \citep{egiazarian2024} adapts multi-codebook quantization from information retrieval \citep{babenko2014additive,jegou2011pq} to LLM compression. PV-tuning \citep{malinovskii2024} extends AQLM with end-to-end fine-tuning of both codebooks and indices using straight-through estimation. Neither work examines how initialisation quality interacts with compression rate or persists through fine-tuning, which is the focus of our paper. This question is particularly important at extreme compression, where poor initialisation can cause additive quantization to fail despite sufficient representational capacity.

\paragraph{EM-based VQ for LLMs.}
Expectation-maximisation (EM) has previously been applied to improve vector quantization. GPTVQ \citep{vanbaalen2024} applies EM-based VQ within the GPTQ framework for LLM compression, using a single learned codebook and performing EM directly for quantization. LSQ++ \citep{martinez2018} similarly applies EM to additive quantization in information retrieval. Our setting differs in two important respects. First, OA-EM operates within additive quantization for LLMs, where multiple codebooks interact combinatorially. Second, OA-EM is output-aware, optimising reconstruction error in activation space and serving specifically as an initialisation stage for downstream beam search rather than as the primary quantizer.

\paragraph{Positioning.}
Existing work improves different stages of the vector quantization pipeline. Structured methods (QuIP\#, QTIP, GLVQ) improve \emph{quantizer design} through codebook geometry; GPTVQ improves the \emph{codebook learning algorithm}; and PV-tuning improves \emph{post-quantization fine-tuning}. 
Our work is orthogonal: we improve the \emph{initialisation stage} and show that codebook initialisation determines the persistent optimisation basin reached by subsequent search and fine-tuning. Rather than proposing a new quantizer geometry, we show that initialisation quality is a previously overlooked bottleneck that dominates performance at extreme compression within free-form additive quantization.
More broadly, the representational ratio $\rho$ and the basin-persistence phenomenon we document apply to any learned-codebook VQ method that relies on greedy sequential initialisation, including future methods combining structured and free-form components.

\section{Method}
\label{sec:method}
We first review the additive quantization framework used in AQLM, then analyse why greedy initialisation becomes a bottleneck at extreme compression. Finally, we introduce OA-EM, an output-aware EM algorithm that improves codebook initialisation and yields better optimisation basins for downstream search and fine-tuning.

\subsection{Background: AQLM}

Additive quantization for LLMs is implemented in AQLM, which represents each weight group $\mathbf{w} \in \mathbb{R}^g$ as a sum of $M$ codewords, $\hat{\mathbf{w}} = \sum_{m=1}^{M} \mathbf{c}_{m, b_m}$,
where $\mathbf{c}_{m,b_m}$ is the $b_m$-th entry of codebook $\mathcal{C}_m \in \mathbb{R}^{K \times g}$, with $K{=}256$. At 2\,bpp with group size $g{=}8$, AQLM uses $M{=}2$ codebooks: each group of 8 weights is represented by two 8-bit indices.

The layer-wise objective minimises output reconstruction error on calibration data, recently shown to be linearly predictive of model perplexity increase \citep{malinovskii-etal-2025-higgs}:
\begin{equation}
    \mathcal{L} = \| \mathbf{X}\mathbf{W} - \mathbf{X}\hat{\mathbf{W}} \|_F^2
    \label{eq:loss}
\end{equation}
where $\mathbf{X}$ is the calibration activation matrix. Optimisation proceeds in two stages. Codebooks are first initialised via residual k-means, after which assignments are refined by beam search over the combinatorial space of codeword combinations.

\subsection{Beam Search: The Standard Remedy}
\label{sec:beam}

Greedy sequential assignment---selecting the best entry from each codebook in turn without revisiting earlier choices---suffers from \emph{premature commitment}: the best first-codebook entry in isolation may pair poorly with available second-codebook entries, but greedy search discards all alternatives, making such errors irrecoverable. AQLM's standard remedy is beam search with width $b$, which defers commitment by maintaining $b$ active candidates at each codebook stage, bridging the gap between greedy search ($b{=}1$) and exhaustive enumeration ($b{=}K^{M-1}$) at a cost of $O(MbK)$ per weight group.

However, wider beams are expensive: on Llama~3.2~3B at 2\,bpp, increasing $b$ from 4 to 16 raises quantization time from 6.1h to 16.9h, a $2.8\times$ cost for a reduction from 352.39 to only 46.01 on WikiText-2 perplexity (Table~\ref{tab:2bpp}). 
More importantly, beam search optimises \emph{assignments} over a fixed set of codebook centroids. If the centroids themselves are poorly placed, even the globally optimal assignment will yield high reconstruction error---beam search finds the best path through a bad tree, but cannot reshape the tree itself. This reveals the deeper limitation we address: the bottleneck is not insufficient search over assignments, but poor codebook geometry. 
OA-EM (\S\ref{sec:oaem}) addresses the root cause directly, improving centroid placement so that even narrow-beam search operates in a favourable landscape.

\subsection{The Initialisation Bottleneck}
\label{sec:theory}

Residual k-means fits codebooks greedily: $\mathcal{C}_1$ is fitted to the weight vectors, then $\mathcal{C}_2$ is fitted to the residuals $\mathbf{w} - \mathbf{c}_{1,b_1}$. This sequential procedure ignores the joint structure; the optimal $\mathcal{C}_1$ depends on what $\mathcal{C}_2$ can represent, and vice versa. We now analyse how this coupling leads to suboptimal assignments and why the effect becomes severe when representational capacity is limited ($\rho > 1$).

\begin{proposition}[Greedy Suboptimality Bound]
\label{prop:greedy}
Let $(i^*, j^*)$ denote the optimal assignment for weight group $\mathbf{w}$, and $(i^g, j^g)$ the greedy sequential assignment. Let $\boldsymbol{\delta} = \mathbf{c}_{1,i^g} - \mathbf{c}_{1,i^*}$ be the first-codebook displacement. The suboptimality gap is:
\begin{align}
    \varepsilon^g - \varepsilon^* &= \underbrace{\|\boldsymbol{\delta}\|^2}_{\text{direct cost}} + \underbrace{2\langle \boldsymbol{\delta},\, \mathbf{c}_{2,j^g} - \mathbf{r}^*\rangle}_{\text{coupling}} \nonumber \\
    &\quad+ \underbrace{\|\mathbf{r}^* - \mathbf{c}_{2,j^g}\|^2 - \varepsilon^*}_{\text{residual mismatch}\;\geq\, 0}
    \label{eq:bound}
\end{align}
where $\mathbf{r}^* = \mathbf{w} - \mathbf{c}_{1,i^*}$ is the joint-optimal residual.
\end{proposition}

\noindent\emph{Proof.} See Appendix~\ref{app:proof}.\medskip

The decomposition reveals three sources of greedy error. The \emph{direct cost} $\|\boldsymbol{\delta}\|^2$ is the squared distance between greedy and optimal first-codebook entries. The \emph{coupling term} captures how well $\mathcal{C}_2$ can compensate for $\boldsymbol{\delta}$---when $\mathcal{C}_2$ has an entry near $\mathbf{r}^* - \boldsymbol{\delta}$, this term can be negative. The \emph{residual mismatch} is always non-negative: $\mathcal{C}_2$ was fitted to the greedy residual, not the optimal one.

\paragraph{Error correction capacity.}
After committing to codebook~1, the remaining codebooks provide $K^{M-1}$ possible residual configurations. At 3\,bpp ($M{=}3$), $K^2 = 65{,}536$ configurations allow the remaining codebooks to absorb most $\boldsymbol{\delta}$ displacements. At 2\,bpp ($M{=}2$), only $K = 256$ entries are available, a $256\times$ reduction in correction capacity. 
Consequently, the second codebook must simultaneously represent the true weight structure \emph{and} compensate for first-codebook errors, making greedy placement much more brittle.
This combinatorial reduction corresponds directly to the representational ratio $\rho$ introduced earlier. When $\rho < 1$, the number of representable code combinations exceeds the number of weight groups, and many greedy errors can be absorbed. When $\rho \gg 1$, weight groups compete for limited capacity, making poor initial placement difficult to fix. Greedy initialisation therefore degrades gracefully at 3\,bpp but can fail catastrophically at 2\,bpp.
Structured methods such as GLVQ \citep{zhang2025glvq} sidestep this problem entirely by replacing free-form codebooks with lattice geometry, eliminating the combinatorial assignment; OA-EM instead solves the problem within the free-form paradigm by improving the initial codebook geometry while preserving the LUT dequantization pathway.

\subsection{OA-EM: Output-Aware EM Initialisation}
\label{sec:oaem}

OA-EM improves upon k-means initialisation by replacing Euclidean distance with Hessian-weighted Mahalanobis distance in both centroid optimisation and code assignment. OA-EM operates within AQLM's residual framework---codebooks are still fitted sequentially---but refines each codebook using output-aware EM that directly targets the reconstruction objective (Eq.~\ref{eq:loss}).

Starting from k-means-initialised centroids $\{\mathbf{c}_k\}_{k=1}^K$ and assignments $\{b_i\}_{i=1}^N$, OA-EM alternates two steps for $R$ rounds:

\smallskip
\noindent\textbf{M-step (centroid optimisation).}\; Fix assignments and optimise centroids to minimise the Hessian-weighted reconstruction error:
\begin{equation}
    \mathcal{L}_{\text{EM}} = \frac{1}{N} \sum_{i=1}^{N} \mathbf{e}_i^\top \mathbf{H}_i\, \mathbf{e}_i, \quad \mathbf{e}_i = \mathbf{w}_i - \mathbf{c}_{b_i}
    \label{eq:mstep}
\end{equation}
where $\mathbf{H}_i = \mathbf{X}_i^\top \mathbf{X}_i + \lambda \mathbf{I}$ is the damped block-diagonal Hessian approximation for weight group~$i$, following the use of second-order information for quantization in GPTQ \citep{frantar2022gptq} and GPTVQ \citep{vanbaalen2024}; the damping constant is $\lambda = 0.01 \cdot \overline{\mathrm{diag}(\mathbf{H})}$.

Centroids are updated via $S$ Adam steps with cosine learning rate annealing from $\eta$ to $0.1\eta$.

\noindent\textbf{E-step (hard reassignment).}\; Fix centroids and reassign each weight group to its nearest centroid under Mahalanobis distance:
\begin{equation}
    b_i \leftarrow \arg\min_{k \in [K]} \; (\mathbf{w}_i - \mathbf{c}_k)^\top \mathbf{H}_i\, (\mathbf{w}_i - \mathbf{c}_k)
    \label{eq:estep}
\end{equation}

\paragraph{Connection to the greedy bound.}
OA-EM does not eliminate sequential fitting, but it directly reduces the dominant error terms in Proposition~\ref{prop:greedy}. Euclidean k-means allocates first-codebook centroids largely according to weight magnitude: large-norm groups attract centroids regardless of their output sensitivity. When $\rho \gg 1$ and representational capacity is scarce, this wastes centroids on output-insensitive groups, producing a large displacement $\boldsymbol{\delta}$ for the output-sensitive groups that dominate the reconstruction loss. OA-EM's Hessian weighting reverses this allocation: the M-step concentrates centroids on groups with large $\mathbf{H}_i$, ensuring small $\boldsymbol{\delta}$ precisely where it matters most---directly reducing the $\|\boldsymbol{\delta}\|^2$ direct cost for groups that contribute most to Eq.~\ref{eq:loss}. Improved first-codebook placement also reduces the residual mismatch term, because the residuals $\mathbf{r} = \mathbf{w} - \mathbf{c}_{1,b_1}$ passed to the second codebook more closely resemble the jointly optimal residuals $\mathbf{r}^*$. We use $R{=}3$ rounds of $S{=}100$ Adam steps with $\eta {=} 10^{-4}$.

\section{Experimental Setup}

\paragraph{Models.} We evaluate on Llama~3.2~3B and Llama~3.1~8B \citep{grattafiori2024llama3}, and Qwen~2.5~3B \citep{qwen2024qwen25}, covering multiple architectures and model sizes. All models are quantized with AQLM at 2\,bpp ($M{=}2$, $g{=}8$). Llama~3.2~3B is additionally evaluated at 3\,bpp ($M{=}3$, $g{=}8$) to test the $\rho$ prediction across compression regimes.

\paragraph{Calibration.} We use 128 sequences from C4 \citep{raffel2020c4} with length 4096 for calibration, following common practice in post-training quantization \citep{frantar2022gptq,lin2024awq}.

\paragraph{Beam search configurations.} We vary beam width $b \in \{4, 8, 16\}$ and maximum epochs $e \in \{5, 100\}$ with early stopping at 0.01 relative MSE, spanning a $2.8\times$ range in quantization time (6.1h to ${\sim}$17h on Llama~3.2~3B). Qwen~2.5~3B and 3\,bpp are evaluated at $b{=}8$, $e{=}100$.

\paragraph{OA-EM configuration.} OA-EM is run for 3 EM rounds with 100 Adam steps per round, learning rate $\eta = 10^{-4}$ with cosine annealing to $0.1\eta$. Both E-step and M-step use the damped block-diagonal Hessian approximation $\mathbf{H}_i = \mathbf{X}_i^\top \mathbf{X}_i + \lambda \mathbf{I}$.

\paragraph{PV-tuning.} We follow the PV-tuning procedure of \citet{malinovskii2024}. Adam optimiser, $\eta_{\text{ft}} = 3 \times 10^{-4}$, batch 32, 10K samples, and 5 epochs; the best WikiText-2 checkpoint is selected. The PV-tuning configuration is identical for all quantization settings, isolating the effect of initialisation.

\paragraph{Evaluation.} We report perplexity on WikiText-2 and C4 (4096 context). Zero-shot evaluation is performed on ARC-Easy and ARC-Challenge \citep{clark2018arc}, HellaSwag \citep{zellers-etal-2019-hellaswag}, PIQA \citep{bisk2020piqa}, WinoGrande \citep{sakaguchi2021winogrande}, LAMBADA \citep{paperno-etal-2016-lambada}, using the LM Evaluation Harness \citep{gao2023lmeval}.

\paragraph{Hardware.} All experiments were run on a single 
NVIDIA A100 80\,GB GPU, except PV-tuning of Llama~3.1~8B, 
which used a single B200 192GB. All results are reported from a single run with a fixed random seed (42).

\section{Initialisation and Basin Persistence}
We study whether codebook initialisation determines the optimisation basin reached by additive quantization. We therefore analyse models before and after PV-tuning. Pre-PV results isolate the effect of initialisation on quantization optimisation, while post-PV results test whether these differences persist after fine-tuning. Persistence would indicate distinct optimisation basins. We conduct this analysis on Llama~3.2~3B as a representative model across a wide range of beam-search configurations.

\begin{table}[t]
\centering
\small

\begin{tabular}{l S S c}
\toprule
 & \textbf{Wiki-2} & \textbf{C4} & \textbf{Time} \\
\midrule
\textbf{FP16} & 7.28 & 11.04 & --- \\
\midrule
\multicolumn{4}{l}{\textbf{Greedy initialisation}} \\
$b{=}4,\ e{=}100$  & 352.39 & 20.66 & 6.1h \\
$b{=}8,\ e{=}100$  & 60.61  & 18.64 & 9.9h \\
$b{=}16,\ e{=}100$ & 46.01  & 19.00 & 16.9h \\
$b{=}8,\ e{=}5$    & 85.72  & 47.04 & 6.8h \\
\midrule
\multicolumn{4}{l}{\textbf{OA-EM initialisation}} \\
$b{=}4,\ e{=}100$  & 16.82 & 18.09 & 6.1h \\
$b{=}8,\ e{=}100$  & \bfseries 17.39 & \bfseries 18.00 & 9.2h \\
$b{=}16,\ e{=}100$ & 16.53 & 17.92 & 15.5h \\
$b{=}8,\ e{=}5$    & 18.91 & 17.98 & 7.3h \\
\bottomrule
\end{tabular}
\caption{Pre-PV-tuning results at 2\,bpp on Llama~3.2~3B. Beam width $b$, maximum epochs $e$. Quantization time on one A100.}
\label{tab:2bpp}
\end{table}

\subsection{Initialisation Effects Before PV-Tuning}
\label{sec:results}

\paragraph{Overcomplete Regime (3\,bpp, $\rho \approx 0.07$).}

At 3\,bpp, the initialisation bottleneck is relatively mild. OA-EM reduces WikiText-2 perplexity from 9.52 to 8.87 ($-$0.65), while C4 increases slightly from 13.39 to 13.51.
LAMBADA accuracy improves from 0.673 to 0.687, while LAMBADA perplexity drops from 4.87 to 4.60. OA-EM reduces quantization time by 5.7\% (12h\,39m vs.\ 13h\,25m), as improved initialisation requires fewer beam-search epochs per layer. Full pre-PV-tuning results are in Appendix~\ref{app:3bpp}.

\paragraph{Undercomplete Regime (2\,bpp, $\rho \approx 18$).}

The 2\,bpp regime reveals the full impact of the initialisation bottleneck. Table~\ref{tab:2bpp} presents pre-PV-tuning results across beam-search configurations. 
Three observations emerge. First, { beam search alone cannot compensate for poor initialisation}. For greedy initialisation, performance improves with wider beams ($b{=}4$: 352.39, $b{=}8$: 60.61, $b{=}16$: 46.01), but remains far from the OA-EM results. In contrast, OA-EM remains stable across beam widths (16.82--17.39).
Second, increasing the search budget does not consistently improve the baseline solution quality. While WikiText-2 perplexity improves with larger beams, C4 perplexity slightly worsens (18.64\,$\to$\,19.00), suggesting overfitting to the calibration objective. Third, OA-EM improves both quality and efficiency. For example, with $b{=}8$, OA-EM achieves 17.39/18.00 in 9.2h compared to 60.61/18.64 in 9.9h for greedy initialisation.

However, pre-PV-tuning results alone do not determine the practical impact of improved initialisation. PV-tuning can substantially improve quantized models, potentially compensating for poor initialisation. We therefore examine whether OA-EM's advantage persists after PV-tuning in \S\ref{sec:postpv}.

\subsection{Basin Persistence After PV-Tuning}
\label{sec:postpv}

PV-tuning \citep{malinovskii2024} performs end-to-end fine-tuning of both codebooks and indices via straight-through estimation. Because PV-tuning substantially improves quantized models, it acts as a strong optimiser that could potentially erase differences caused by initialisation. A key question is therefore whether improved initialisation still leads to better final models, or whether PV-tuning causes all configurations to converge to the same solution. Table~\ref{tab:pvtuning} reports WikiText-2 perplexity before and after PV-tuning across all quantization configurations on Llama~3.2~3B. PV-tuning substantially improves all models, dramatically reducing the gap caused by poor initialisation.
However, OA-EM consistently achieves lower final perplexity across every configuration. Thus, although PV-tuning mitigates poor initialisation, it does not eliminate its effect: models starting from better initial codebooks converge to better final solutions.

\begin{table}[t]
\centering
\small

\begin{tabular}{l@{\hskip 4pt}l@{\hskip 8pt}cc@{\hskip 8pt}cr}
\toprule
Init & Config & Pre-PV & Post-PV & $\Delta$ & Q-time \\
\midrule
\multicolumn{6}{l}{\textit{Narrow beam ($b{=}4$, $e{=}100$)}} \\
Greedy & & 352.39 & 12.66 & $-$339.73 & 6.1h \\
OA-EM & & 16.82 & \textbf{11.53} & $-$5.29 & 6.1h \\
\midrule
\multicolumn{6}{l}{\textit{Standard beam ($b{=}8$, $e{=}100$)}} \\
Greedy & & 60.61 & 11.76 & $-$48.85 & 9.9h \\
OA-EM & & 17.39 & \textbf{11.53} & $-$5.86 & 9.2h \\
\midrule
\multicolumn{6}{l}{\textit{Wide beam ($b{=}16$, $e{=}100$)}} \\
Greedy & & 46.01 & 12.01 & $-$34.00 & 16.9h \\
OA-EM & & 16.53 & \textbf{11.49} & $-$5.04 & 15.5h \\
\midrule
\multicolumn{6}{l}{\textit{Early stopping ($b{=}8$, $e{=}5$)}} \\
Greedy & & 85.72 & 12.69 & $-$73.03 & 6.8h \\
OA-EM & & 18.91 & \textbf{11.76} & $-$7.15 & 7.3h$^\dagger$ \\
\bottomrule
\end{tabular}
\caption{WikiText-2 perplexity before and after PV-tuning at 2\,bpp on Llama~3.2~3B.  
}
\label{tab:pvtuning}
\vspace{0.5em}
{\footnotesize † The only configuration in which OA-EM quantization exceeds the greedy baseline.}
\end{table}

\paragraph{Asymmetric beam-width scaling.}
The response to the beam width differs between the two initialisations. The greedy baseline is non-monotonic: $b{=}4$ (12.66) $\rightarrow$ $b{=}8$ (11.76, best) $\rightarrow$ $b{=}16$ (12.01). In contrast, OA-EM remains stable at narrow beams and improves slightly with wider search: $b{=}4$ (11.53) $\rightarrow$ $b{=}8$ (11.53) $\rightarrow$ $b{=}16$ (11.49). If both methods converged to the same solution after PV-tuning, the beam width would affect them similarly. The contrasting responses suggest that the two initialisations lead to different optimisation trajectories.

\paragraph{Pareto improvements.}
OA-EM also improves the quality–time trade-off. For example, OA-EM at $b{=}4$ (6.1h, 11.53 ppl) outperforms greedy at $b{=}8$ (9.9h, 11.76 ppl), achieving better perplexity with 38\% less quantization time. OA-EM with early stopping ($b{=}8$, $e{=}5$, 7.3h) matches the full greedy baseline ($b{=}8$, $e{=}100$, 9.9h) in 26\% less time. The only configuration where OA-EM is slower is the early-stopping setting, where the fixed cost of OA-EM initialisation is not fully amortised. Overall, the cheapest OA-EM run produces a better final model than the most expensive greedy run. A full Pareto analysis is provided in Appendix~\ref{app:pareto}.

\section{Downstream Task Performance}
\label{sec:downstream}

Table~\ref{tab:downstream_summary} summarises downstream performance across all models and beam configurations. OA-EM matches or improves average accuracy in every setting where the initialisation bottleneck is non-trivial. Full per-task breakdowns, including LAMBADA (the downstream task most sensitive to perplexity improvements), are in Appendix~\ref{app:downstream}.

\begin{table}[t]
\centering
\small

\setlength{\tabcolsep}{4pt}
\begin{tabular}{@{}llcc@{}}
\toprule
& & \multicolumn{2}{c}{Avg Acc$\uparrow$} \\
Model & Config & Grdy & OA-EM \\
\midrule
Llama 3B & $b{=}4$ & .573 & \textbf{.589} \\
Llama 3B & $b{=}8$ & .585 & \textbf{.591} \\
Llama 3B & $b{=}16$ & .585 & \textbf{.591} \\
Llama 3B & $b{=}8,e{=}5$ & .563 & \textbf{.572} \\
\midrule
Llama 8B & $b{=}8$ & .649 & \textbf{.656} \\
\midrule
Qwen 3B & $b{=}8$ & \textbf{.606} & .603 \\
\bottomrule
\end{tabular}
\caption{Post-PV-tuning downstream summary at 2\,bpp. Avg = mean of 6 zero-shot accuracy tasks. Full per-task tables in Appendix~\ref{app:downstream}.}
\label{tab:downstream_summary}
\end{table}

On Llama~3.2~3B, OA-EM wins average accuracy across all four beam configurations, with the clearest gains at $b{=}4$ (+1.7pp) where the greedy baseline has minimal search to compensate for poor initialisation. On Llama~3.1~8B, OA-EM wins 4 of 6 accuracy tasks with a 0.7-point average improvement. While OA-EM establishes Pareto dominance in perplexity across all architectures, the baseline holds a small downstream advantage on Qwen~2.5~3B (0.606 vs.\ 0.603 average accuracy), consistent with the mild initialisation bottleneck on this model (comparable $\rho$, but smoother weight statistics; see \S\ref{sec:cross}). The post-PV perplexity gap of 0.20 points is precisely preserved from the pre-PV gap, providing clear evidence of distinct optimisation basins even when the bottleneck is mild. At the 3B scale, zero-shot evaluations are inherently high-variance \citep{gao2023lmeval}; perplexity remains the more reliable signal \citep{egiazarian2024,frantar2022gptq,tseng2024quipsharp}, and on this metric OA-EM wins on both WikiText-2 and C4 after PV-tuning across every model we test.

\section{Generality Across Compression Rates and Architectures}
\label{sec:cross}

Table~\ref{tab:summary} summarises perplexity results across compression rates and architectures.
\paragraph{The $\rho$ gradient.}
The representational ratio $\rho$ governs the \emph{severity}, rather than the \emph{existence}, of the initialisation bottleneck. At 3\,bpp ($\rho \approx 0.07$), the pre-PV gap of 0.65 compresses only $5.4\times$ to 0.12 through PV-tuning---the gap attenuates but does not vanish, and OA-EM wins 5/6 downstream tasks after PV-tuning, and improves ARC-Easy by 3.5 points (Appendix~\ref{app:3bpp_pv}). At 2\,bpp ($\rho \approx 18$), the dramatic 43-point pre-PV gap compresses $188\times$ to 0.23, yet OA-EM wins every beam configuration and most downstream tasks. Even in the overcomplete regime, where sufficient codebook capacity exists to absorb initialisation errors, OA-EM's output-aware placement finds a solution that is measurably better and is preserved after PV-tuning. These results suggest that PV-tuning improves models within their existing optimisation basin rather than moving them between basins.

\begin{table}[!htb]
\centering
\small

\setlength{\tabcolsep}{3.5pt}
\begin{tabular}{@{}llcccc@{}}
\toprule
& & \multicolumn{2}{c}{Pre-PV} & \multicolumn{2}{c}{Post-PV} \\
Model & Rate & Grdy & OA-EM & Grdy & OA-EM \\
\midrule
Llama 3B & 3\,bpp & 9.52 & \textbf{8.87} & 8.66 & \textbf{8.54} \\
\midrule
Llama 3B & 2\,bpp & 60.61 & \textbf{17.39} & 11.76 & \textbf{11.53} \\
Llama 8B & 2\,bpp & 18.86 & \textbf{16.38} & 9.39 & \textbf{9.25} \\
Qwen 3B & 2\,bpp & 12.50 & \textbf{12.30} & 10.93 & \textbf{10.73} \\
\bottomrule
\end{tabular}
\caption{Basin persistence across compression rates and architectures ($b{=}8$, $e{=}100$). WikiText-2 perplexity before and after PV-tuning.}
\label{tab:summary}
\end{table}

\paragraph{Why gap compression varies across models.}
The 8B model provides an informative counterpoint. Despite having \emph{higher} $\rho$ than 3B (wider layers, same $K^2 = 65{,}536$ capacity), the 8B baseline (18.86) is dramatically better than the 3B baseline (60.61). This reveals that $\rho$ is necessary but not sufficient to predict catastrophic failure; the weight distribution also matters. This is consistent with scaling law analysis showing that larger models exhibit less quantization-induced degradation \citep{ouyang-etal-2025-low}. We attribute the 8B's resilience to smoother per-layer weight statistics: Llama~3.1~8B, trained on substantially more data (15T vs.\ 3T tokens), has fewer high-magnitude outlier groups that dominate first-codebook capacity under greedy placement. When weight statistics are smoother, greedy initialisation is less catastrophic even under high $\rho$, though measurably suboptimal.

\begin{table}[!htb]
\centering
\small

\setlength{\tabcolsep}{3pt}
\begin{tabular}{@{}lcccc@{}}
\toprule
Benchmark & Domain & Bsln & OA-EM & Ratio \\
\midrule
C4 & In-domain & 18.64 & 18.00 & 1.04$\times$ \\
LAMBADA & Near-OOD & 12.28 & 8.85 & 1.39$\times$ \\
WikiText-2 & Far-OOD & 60.61 & 17.39 & 3.49$\times$ \\
\bottomrule
\end{tabular}
\caption{Degradation scales with domain distance at 2\,bpp ($b{=}8$, pre-PV-tuning). Ratio = baseline/OA-EM.}
\label{tab:domain}
\end{table}
\section{Domain-Dependent Degradation}
\label{sec:domain}

The baseline's pre-PV failure is not uniform; it scales with domain distance from the C4 calibration set (Table~\ref{tab:domain}). We observe a gradient from 1.04$\times$ degradation in-domain (C4) to 3.49$\times$ far out-of-domain (WikiText-2). The mechanism follows from the undercomplete regime. The layer-wise objective (Eq.~\ref{eq:loss}) implicitly weights each weight group by its calibration importance. When $\rho \gg 1$, limited capacity is concentrated on calibration-important groups, producing codebooks that partially memorise calibration-specific statistics. Under a different evaluation domain,
the importance profile of weight groups shifts: groups that were unimportant during calibration may become important for the new distribution. Because these groups received little representational capacity during quantization, reconstruction error grows with domain divergence. OA-EM mitigates this effect because its output-aware objective distributes codebook capacity according to Hessian-weighted output sensitivity rather than calibration frequency alone. As a result, the learned codebooks better preserve the linguistic representations that remain important across domains, improving robustness to domain shift.

\section{Conclusion}
This work shows that codebook initialisation determines the optimisation basin for additive quantization, with severity governed by the representational ratio $\rho$. In the undercomplete regime ($\rho \gg 1$), greedy initialisation traps the model in a suboptimal basin that beam search and PV-tuning struggle to escape.
In the overcomplete regime ($\rho < 1$), the effect attenuates but does not disappear: even with sufficient representational capacity, initial placement continues to influence the final solution.
Across compression rates, search budgets, and model architectures, OA-EM consistently produces lower perplexity than greedy initialisation after PV-tuning. A particularly notable observation is that beam width can have \emph{opposite} effects depending on the initialisation, behaviour that is consistent with optimisation proceeding within different basins rather than reliably moving between them.
The practical implication is immediate. OA-EM at beam\,4 (6.1h) produces a better model—on both perplexity and downstream tasks—than the greedy baseline at beam\,16 (16.9h), a $2.8\times$ speedup. More broadly, our results suggest a simple principle for extreme compression: improving initialisation quality may be more effective than increasing search intensity.
Finally, the representational ratio $\rho$ provides a useful lens for understanding when additive quantization becomes brittle. As LLM deployment increasingly targets edge and CPU environments where LUT-based dequantization is essential, improving the optimisation geometry of learned-codebook quantizers may be as important as designing new quantization schemes.

\section*{Limitations}

We evaluate on three models from two architecture families (Llama~3.2~3B, Llama~3.1~8B, Qwen~2.5~3B). While basin persistence holds across all three, the downstream accuracy signal is clear only on Llama~3.2~3B where the initialisation bottleneck is most severe; on Qwen, the baseline slightly leads on average accuracy (0.606 vs.\ 0.603). We focus on the 3B--8B parameter range as the regime most relevant for consumer GPU and edge deployment; larger models would strengthen the generality claim but require multi-GPU infrastructure. Our method applies to free-form additive quantization and does not directly transfer to lattice-based (QuIP\#, GLVQ) or trellis-based (QTIP) methods, which avoid the discrete assignment problem by constraining codebook geometry; however, the $\rho$ framework and basin persistence analysis may inform initialisation strategies for future hybrid approaches that combine structured and learned components. We do not compare absolute perplexity against GLVQ or QTIP, as our contribution is orthogonal: we improve initialisation within the free-form paradigm rather than proposing an alternative codebook geometry. We evaluate English-only models and benchmarks.

\section*{Ethical Considerations}

This work improves the efficiency of existing open-weight LLMs. We do not introduce new training data or models. Compression reduces deployment costs and environmental impact but does not address biases present in the original models.

\section*{Acknowledgements}

This work was supported by the UKRI AI Centre for Doctoral Training in Speech and Language Technologies (SLT) and their Applications funded by UK Research and Innovation [grant number EP/S023062/1]. For the purpose of open access, the author has applied a Creative Commons Attribution (CC BY) licence to any Author Accepted Manuscript version arising.

We acknowledge IT Services at The University of Sheffield for the provision of services for High Performance Computing.


\bibliography{anthology-1,anthology-2,custom}
\appendix

\section{Overcomplete Regime: 3\,bpp Detailed Results}
\label{app:3bpp}

Table~\ref{tab:3bpp} presents the pre-PV-tuning perplexity at 3\,bpp on Llama~3.2~3B.

\begin{table}[ht!]
\centering
\small
\caption{Pre-PV-tuning perplexity at 3\,bpp on Llama~3.2~3B ($b{=}8$, $e{=}100$).}
\label{tab:3bpp}
\setlength{\tabcolsep}{10pt}
\begin{tabular}{lcc}
\toprule
& Wiki-2$\downarrow$ & C4$\downarrow$ \\
\midrule
FP16 & 7.28 & 11.04 \\
\midrule
Baseline & 9.52 & \textbf{13.39} \\
+\,OA-EM & \textbf{8.87} & 13.51 \\
\bottomrule
\end{tabular}
\end{table}

\subsection{3\,bpp Post-PV-Tuning Downstream}
\label{app:3bpp_pv}

After PV tuning, the gap at 3~bpp shrinks from 0.65 to 0.12 (WikiText-2: 8.66 vs.\ 8.54; C4: 11.43 vs.\ 11.45). Table~\ref{tab:3bpp_pvdown} reports the full post-PV downstream evaluation. OA-EM wins 4 of 6 tasks and ties on 1, with the largest gains on ARC-Easy (+3.5pp) and LAMBADA accuracy (+1.6pp). It also achieves the lowest LAMBADA perplexity. Average accuracy increases by 0.7 percentage points (0.647,$\to$,0.654), with the only notable regression on WinoGrande ($-1.2$pp). These results indicate that basin persistence extends to the overcomplete regime.

\section{Full Downstream and Perplexity Results}
\label{app:downstream}

This appendix provides the complete per-task downstream evaluations and per-model perplexity breakdowns that underlie the summary in Table~\ref{tab:downstream_summary} and Table~\ref{tab:summary}. Results are organised by model.

\subsection{Llama~3.2~3B at 2\,bpp}
Tables~\ref{tab:pvdownstream4}--\ref{tab:pvdownstream5} present post-PV-tuning downstream results across all four beam configurations. OA-EM wins or ties average accuracy in every setting; the advantage is largest at $b{=}4$ (+1.7pp), where the greedy baseline has minimal search to compensate for poor initialisation.

\subsection{Llama~3.1~8B at 2\,bpp}

Table~\ref{tab:pv8b} shows the perplexity breakdown before and after PV-tuning. Table~\ref{tab:downstream8b} presents post-PV-tuning downstream results.

\subsection{Qwen~2.5~3B at 2\,bpp}

Table~\ref{tab:pvqwen} presents the perplexity breakdown; OA-EM wins on both metrics after PV-tuning despite the mild bottleneck on this architecture. Table~\ref{tab:downstreamqwen} shows the downstream evaluation where the baseline nominally leads on average accuracy (0.606 vs.\ 0.603), as discussed in \S\ref{sec:downstream}.

\section{Pareto Analysis}
\label{app:pareto}

Table~\ref{tab:pareto} consolidates the quality-compute frontier at 2\,bpp on Llama~3.2~3B. Every OA-EM configuration produces a better final model than any greedy configuration at equal or lower compute.

\section{Proof of Proposition~\ref{prop:greedy}}
\label{app:proof}

\begin{proof}
Let $\mathbf{r}^g = \mathbf{w} - \mathbf{c}_{1,i^g} = \mathbf{r}^* - \boldsymbol{\delta}$. Then:
\begin{align}
    \varepsilon^g &= \|\mathbf{r}^g - \mathbf{c}_{2,j^g}\|^2 = \|\mathbf{r}^* - \boldsymbol{\delta} - \mathbf{c}_{2,j^g}\|^2 \nonumber \\
    &= \|\mathbf{r}^* - \mathbf{c}_{2,j^g}\|^2 + \|\boldsymbol{\delta}\|^2 - 2\langle \mathbf{r}^* - \mathbf{c}_{2,j^g},\, \boldsymbol{\delta}\rangle
\end{align}
Subtracting $\varepsilon^* = \|\mathbf{r}^* - \mathbf{c}_{2,j^*}\|^2$ and rearranging yields Eq.~\eqref{eq:bound}. The residual mismatch $\|\mathbf{r}^* - \mathbf{c}_{2,j^g}\|^2 - \|\mathbf{r}^* - \mathbf{c}_{2,j^*}\|^2 \geq 0$ since $j^*$ minimises over $\mathbf{r}^*$.
\end{proof}

\begin{table*}[ht!]
\centering
\small
\caption{Post-PV-tuning downstream at 3\,bpp on Llama~3.2~3B ($b{=}8$, $e{=}100$). Avg = mean of 6 accuracy tasks.}
\label{tab:3bpp_pvdown}
\begin{tabular}{lccccccc|c}
\toprule
& ARC-C & ARC-E & HeSw & LAM\,acc & LAM\,ppl$\downarrow$ & PIQA & WiGr & Avg$\uparrow$ \\
\midrule
Greedy & .418 & .658 & .704 & .670 & 4.87 & .760 & \textbf{.669} & .647 \\
OA-EM & \textbf{.421} & \textbf{.693} & .704 & \textbf{.685} & \textbf{4.64} & \textbf{.762} & .657 & \textbf{.654} \\
\bottomrule
\end{tabular}
\end{table*}

\begin{table*}[ht!]
\centering
\small
\caption{Post-PV downstream, Llama~3.2~3B, 2\,bpp, $b{=}4$, $e{=}100$. OA-EM wins or ties every metric.}
\label{tab:pvdownstream4}
\begin{tabular}{lccccccc|c}
\toprule
& ARC-C & ARC-E & HeSw & LAM\,acc & LAM\,ppl$\downarrow$ & PIQA & WiGr & Avg$\uparrow$ \\
\midrule
Greedy & .350 & .560 & .620 & .577 & 7.65 & .734 & .594 & .573 \\
OA-EM & \textbf{.359} & \textbf{.614} & \textbf{.625} & \textbf{.584} & \textbf{7.13} & .734 & \textbf{.619} & \textbf{.589} \\
\bottomrule
\end{tabular}
\end{table*}

\begin{table*}[ht!]
\centering
\small
\caption{Post-PV downstream, Llama~3.2~3B, 2\,bpp, $b{=}8$, $e{=}100$.}
\label{tab:pvdownstream}
\begin{tabular}{lccccccc|c}
\toprule
& ARC-C & ARC-E & HeSw & LAM\,acc & LAM\,ppl$\downarrow$ & PIQA & WiGr & Avg$\uparrow$ \\
\midrule
Greedy & .356 & .600 & .623 & .574 & 7.96 & \textbf{.741} & \textbf{.614} & .585 \\
OA-EM & \textbf{.366} & \textbf{.601} & \textbf{.626} & \textbf{.604} & \textbf{6.95} & .736 & .611 & \textbf{.591} \\
\bottomrule
\end{tabular}
\end{table*}

\begin{table*}[ht!]
\centering
\small
\caption{Post-PV downstream, Llama~3.2~3B, 2\,bpp, $b{=}16$, $e{=}100$.}
\label{tab:pvdownstream16}
\begin{tabular}{lccccccc|c}
\toprule
& ARC-C & ARC-E & HeSw & LAM\,acc & LAM\,ppl$\downarrow$ & PIQA & WiGr & Avg$\uparrow$ \\
\midrule
Greedy & \textbf{.364} & .619 & .625 & .579 & \textbf{7.52} & \textbf{.731} & .589 & .585 \\
OA-EM & .357 & \textbf{.624} & \textbf{.627} & \textbf{.587} & 7.66 & .726 & \textbf{.624} & \textbf{.591} \\
\bottomrule
\end{tabular}
\end{table*}

\begin{table*}[ht!]
\centering
\small
\caption{Post-PV downstream, Llama~3.2~3B, 2\,bpp, $b{=}8$, $e{=}5$ (early stopping).}
\label{tab:pvdownstream5}
\begin{tabular}{lccccccc|c}
\toprule
& ARC-C & ARC-E & HeSw & LAM\,acc & LAM\,ppl$\downarrow$ & PIQA & WiGr & Avg$\uparrow$ \\
\midrule
Greedy & \textbf{.332} & \textbf{.569} & .611 & .556 & 8.61 & .716 & .594 & .563 \\
OA-EM & .331 & .561 & \textbf{.625} & \textbf{.574} & \textbf{7.73} & \textbf{.734} & \textbf{.609} & \textbf{.572} \\
\bottomrule
\end{tabular}
\end{table*}

\clearpage

\begin{table*}[ht!]
\centering
\small
\caption{WikiText-2 / C4 perplexity before and after PV-tuning at 2\,bpp on Llama~3.1~8B ($b{=}8$, $e{=}100$).}
\label{tab:pv8b}
\setlength{\tabcolsep}{10pt}
\begin{tabular}{lcccc}
\toprule
& \multicolumn{2}{c}{Before PV} & \multicolumn{2}{c}{After PV} \\
Init & Wiki-2 & C4 & Wiki-2 & C4 \\
\midrule
Greedy & 18.86 & 15.01 & 9.39 & 12.02 \\
OA-EM & \textbf{16.38} & \textbf{14.50} & \textbf{9.25} & \textbf{11.89} \\
\bottomrule
\end{tabular}
\end{table*}

\begin{table*}[ht!]
\centering
\small
\caption{Post-PV downstream, Llama~3.1~8B, 2\,bpp, $b{=}8$, $e{=}100$.}
\label{tab:downstream8b}
\begin{tabular}{lccccccc|c}
\toprule
& ARC-C & ARC-E & HeSw & LAM\,acc & LAM\,ppl$\downarrow$ & PIQA & WiGr & Avg$\uparrow$ \\
\midrule
Greedy & \textbf{.432} & .656 & .707 & \textbf{.679} & 4.60 & .759 & .661 & .649 \\
OA-EM & .424 & \textbf{.677} & \textbf{.714} & .675 & \textbf{4.59} & \textbf{.769} & \textbf{.679} & \textbf{.656} \\
\bottomrule
\end{tabular}
\end{table*}

\begin{table*}[t!] 
\centering
\small
\caption{WikiText-2 / C4 perplexity before and after PV-tuning at 2\,bpp on Qwen~2.5~3B ($b{=}8$, $e{=}100$).}
\label{tab:pvqwen}
\setlength{\tabcolsep}{10pt}
\begin{tabular}{lcccc}
\toprule
& \multicolumn{2}{c}{Before PV} & \multicolumn{2}{c}{After PV} \\
Init & Wiki-2 & C4 & Wiki-2 & C4 \\
\midrule
Greedy & 12.50 & \textbf{16.01} & 10.93 & 14.57 \\
OA-EM & \textbf{12.30} & 16.08 & \textbf{10.73} & \textbf{14.49} \\
\bottomrule
\end{tabular}
\end{table*}

\begin{table*}[t!] 
\centering
\small
\caption{Post-PV downstream, Qwen~2.5~3B, 2\,bpp, $b{=}8$, $e{=}100$.}
\label{tab:downstreamqwen}
\begin{tabular}{lccccccc|c}
\toprule
& ARC-C & ARC-E & HeSw & LAM\,acc & LAM\,ppl$\downarrow$ & PIQA & WiGr & Avg$\uparrow$ \\
\midrule
Greedy & \textbf{.375} & \textbf{.662} & .626 & \textbf{.600} & \textbf{7.29} & \textbf{.739} & .634 & \textbf{.606} \\
OA-EM & .366 & .651 & \textbf{.630} & .587 & 7.37 & .737 & \textbf{.647} & .603 \\
\bottomrule
\end{tabular}
\end{table*}

\begin{table*}[ht!]
\centering
\small
\caption{Quality-compute Pareto analysis at 2\,bpp on Llama~3.2~3B. All metrics post-PV-tuning. Avg = mean of 6 accuracy tasks.}
\label{tab:pareto}
\begin{tabular}{llccr}
\toprule
Init & Config & Wiki-2$\downarrow$ & Avg$\uparrow$ & Q-time \\
\midrule
Greedy & $b{=}4$, $e{=}100$ & 12.66 & .573 & 6.1h \\
OA-EM & $b{=}4$, $e{=}100$ & 11.53 & .589 & 6.1h \\
\midrule
Greedy & $b{=}8$, $e{=}100$ & 11.76 & .585 & 9.9h \\
OA-EM & $b{=}8$, $e{=}100$ & 11.53 & .591 & 9.2h \\
\midrule
Greedy & $b{=}16$, $e{=}100$ & 12.01 & .585 & 16.9h \\
OA-EM & $b{=}16$, $e{=}100$ & \textbf{11.49} & \textbf{.591} & 15.5h \\
\midrule
Greedy & $b{=}8$, $e{=}5$ & 12.69 & .563 & 6.8h \\
OA-EM & $b{=}8$, $e{=}5$ & 11.76 & .572 & 7.3h \\
\bottomrule
\end{tabular}
\end{table*}
\end{document}